\documentclass[journal]{IEEEtran}  

\usepackage{amsmath,amssymb,amsthm,bm,mathtools}
\usepackage{booktabs}
\usepackage{multirow}
\usepackage{rotating}
\usepackage{algorithm}
\usepackage{algorithmic}
\usepackage{graphicx}
\usepackage[hidelinks]{hyperref}
\usepackage{xcolor}

\newtheorem{theorem}{Theorem}
\newtheorem{lemma}{Lemma}
\newtheorem{proposition}{Proposition}
\newtheorem{corollary}{Corollary}
\theoremstyle{definition}

\newtheorem{assumption}{Assumption}
\newtheorem{remark}{Remark}

\title{\LARGE \bf From Optical Flow to Safety: Visual Safety Certification via Scale-Free Relative Motion}
\author{Bazeela Banday and Jishnu Keshavan%
\thanks{ The authors are with the Department of Mechanical Engineering, Indian Institute of Science. \tt\small \{bazeelab, kjishnu\}@iisc.ac.in}
}

\begin{document}
\maketitle
\thispagestyle{empty}
\pagestyle{empty}

\begin{abstract}
Safe autonomous navigation using monocular vision remains challenging
because explicit obstacle states are not directly observable. This
letter introduces a visual safety-certification paradigm that uses
scale-free, depth-normalized relative motion encoded by monocular
optical flow as the representation for collision-risk certification,
without reconstructing obstacle position, velocity, geometry, or a
metric world model. Unlike conventional optical-flow-based navigation approaches that rely on image motion to shape avoidance behavior, the proposed framework directly encodes collision-cone geometry and safety constraints in this relative-motion representation. An online Concurrent
Learning observer estimates the normalized relative velocity and
provides a computable estimation-error bound. This uncertainty is
propagated into the collision cone through angular inflation and a
monocular robot-footprint margin, yielding a Robust Vision-Based
Collision Cone Control Barrier Function (RV-C$^3$BF). An optimal-decay
CBF quadratic program minimally modifies the nominal control input to
enforce the resulting robust safety constraint under the stated
assumptions. Extensive simulation and real-world experiments on a TurtleBot3 in cluttered environments with static and dynamic obstacles demonstrate collision avoidance using
only monocular vision and onboard proprioception, with the estimation-error bound containing the measured error in $98.3\%$ of simulation samples and the robot achieving collision-free navigation in $93$ of $100$ real-world trials.
\end{abstract}


\section{Introduction}
\label{sec:intro}

Safe autonomous navigation in dynamic and unstructured environments
requires robots to reason about collision risk despite incomplete
knowledge of their surroundings. Monocular cameras offer a lightweight
sensing modality for mobile robots but do not directly provide metric
depth or explicit obstacle states. Optical flow, however, encodes
scene motion relative to the camera without requiring explicit
three-dimensional reconstruction \cite{srinivasan2000honeybees}.
This raises a fundamental question: can collision risk be formally
certified directly from optical flow, without reconstructing a metric
representation of the environment?

Optical flow has long been used to guide collision-avoidance
behavior in autonomous systems
\cite{deng2026autonomous,hu2025seeing}. However, its use as a
representation for formal safety certification remains largely
unexplored. Existing optical-flow-based methods primarily use image
motion to generate or condition avoidance actions, rather than to
construct explicit safety constraints. This work addresses this gap by
introducing a visual safety-certification paradigm that bypasses
obstacle-state reconstruction and constructs collision geometry and
robust safety constraints directly from the scale-free relative motion
encoded by optical flow.

Control Barrier Functions (CBFs) provide a formal framework for
enforcing safety through optimization-based control filters
\cite{ames2019control}. Collision Cone Control Barrier Functions
(C$^3$BFs) characterize collision risk through relative-motion
direction \cite{tayal2026collision}, but typically rely on explicit
obstacle states or geometric information. Recent vision-based CBF
approaches use RGB-D sensing \cite{abdi2023safe}, dense point clouds
\cite{de2024point}, vision-language models \cite{chen2026dynamic},
or learned multi-camera perception \cite{kapoor2025visafe}.
Measurement-robust CBFs account for perception uncertainty
\cite{cosner2021measurement,cosner2022self}, but require uncertainty
bounds defined over an explicit measurement representation. Thus, existing approaches either rely on explicit geometric models or use image motion without formal safety certification.

\begin{figure}[t]
    \centering
    \includegraphics[width=\columnwidth]{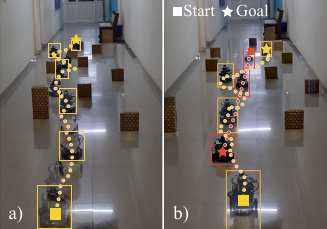}
    \caption{\textbf{Monocular vision-based safety certification in real-world navigation.}
    Collision geometry is constructed from scale-free relative motion obtained
    through monocular optical flow and onboard proprioception, without explicit
    obstacle-state reconstruction. (a) Navigation through a corridor with static
    obstacles. (b) Navigation among static obstacles and a moving obstacle.
    Yellow: robot; red: moving obstacle; square: start; star: goal.}
    \label{fig:main}
\end{figure}

The key insight is that collision-cone reasoning fundamentally depends
on relative-motion direction rather than metric obstacle state. Under
perspective projection, translational optical flow is directly related
to depth-normalized relative velocity, preserving its relative-motion
direction while eliminating metric scale. Together with the
line-of-sight direction, these quantities provide the information
required to construct collision-cone geometry. Thus, monocular optical
flow, combined with onboard proprioception, provides a scale-free
representation in which collision risk can be certified directly,
without reconstructing obstacle position, velocity, geometry, or
a metric world model. Rather than treating optical flow merely as a cue for
avoidance behavior, we use it as the underlying representation for
formal collision-risk certification. Instead of replacing one state estimate with another, we show that collision-cone safety certification can be performed directly in the monocular visual motion representation.

Building on this representation, we develop the estimation and safety
machinery required to make the certification robust to visual estimation
errors. An online Concurrent Learning (CL) observer \cite{chowdhary2010concurrent} estimates the
depth-normalized relative velocity and provides a computable
estimation-error bound. This uncertainty is propagated into the
collision-cone geometry through angular inflation and combined with a
monocular robot-footprint margin, yielding the Robust Vision-Based
Collision Cone Control Barrier Function (RV-C$^3$BF). An optimal-decay
CBF quadratic program (QP) \cite{zeng2021optimal} then enforces the
resulting robust safety constraint while minimally modifying the nominal
control input. The complete framework is implemented onboard a TurtleBot3
and validated in cluttered environments containing both static and
dynamic obstacles. 

The main contributions are:
\begin{itemize}

    \item  A visual safety-certification paradigm showing that
    scale-free relative motion, naturally encoded by monocular optical
    flow, suffices for collision-cone safety certification without
    explicit obstacle-state or geometric world-model reconstruction.

    \item Online estimation of depth-normalized relative velocity directly from image measurements, with a computable error bound that feeds the safety certificate.

    \item The RV-C$^3$BF that propagates representation uncertainty and a
    monocular robot-footprint margin into the collision-cone constraint
    for robust safety enforcement.

    \item Extensive simulation and real-world validation on a TurtleBot3 navigating cluttered environments populated with static and dynamic obstacles using only monocular vision and onboard proprioception.
\end{itemize}

\section{Preliminaries and Background}
\label{sec:prelim}



\subsection{Vision Perception Pipeline}
Let $I(\boldsymbol{p},t)$ denote image intensity at pixel
$\boldsymbol{p}=[u,v]^\top$. Brightness constancy yields the
optical-flow constraint
$I_u f_u+I_v f_v+I_t=0$, where
$\boldsymbol{f}_{\mathrm{total}}=[f_u,f_v]^\top$ is the optical flow, estimated using the Lucas--Kanade algorithm \cite{lucas1981iterative}.

\begin{figure*}[t]
\includegraphics[width=\textwidth]{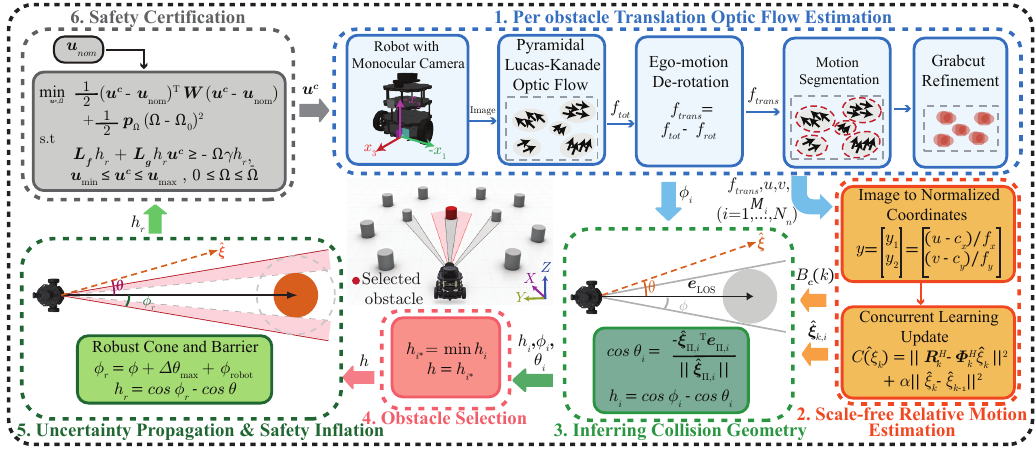}
\caption{\textbf{Monocular vision-based perception-to-control pipeline.}
De-rotated optical flow and onboard proprioception yield scale-free relative motion, estimated online using Concurrent Learning (CL). The nominal V-C$^3$BF identifies the most critical collision cone,
which is robustified using the CL estimation-error bound and monocular footprint margin to form the RV-C$^3$BF. An optimal-decay CBF-QP enforces the resulting safety constraint. The resulting loop performs safety-critical navigation using
monocular vision and onboard proprioception, with collision-cone safety
certification performed directly in the native geometry of the sensing
representation, without explicit reconstruction of obstacle position,
velocity, or geometry.}
\label{fig:BD}
\end{figure*}

The perception pipeline extracts image-space motion for collision-cone
construction from monocular image sequences. Camera rotation is
compensated, independently moving regions are identified, and obstacle
tracks are maintained for relative-motion estimation.

Pyramidal Lucas--Kanade estimates optical flow between consecutive
frames. The rotation-induced component, computed from onboard
proprioception, is removed as $\boldsymbol{f}_{\mathrm{trans}}
=
\boldsymbol{f}_{\mathrm{total}}
-
\boldsymbol{f}_{\mathrm{rot}},$
where $\boldsymbol{f}_{\mathrm{rot}}$ denotes the
rotation-induced optical flow. The compensated flow is clustered using DBSCAN \cite{ester1996density}
to identify candidate obstacle regions. GrabCut \cite{rother2004grabcut}
refines these regions into dense masks, which are associated across
frames using an image-space tracker.

For each tracked obstacle $i$, the perception module outputs $\mathcal P_i
=
\left\{
\boldsymbol{p}_i,\,
\boldsymbol{f}_{\mathrm{trans},i},\,
\mathcal M_i,\,
\mathcal T_i
\right\},$ where $\boldsymbol{p}_i$ denotes the tracked pixel locations,
$\boldsymbol{f}_{\mathrm{trans},i}$ the compensated flow,
$\mathcal{M}_i$ the obstacle mask, and $\mathcal{T}_i$ the persistent
track identifier (Fig.~\ref{fig:BD}).  Camera intrinsics map the pixel
locations and compensated flow to the normalized image coordinates and
velocities used by the CL observer as
described in Sec.~\ref{sec:cl}. Thus, the perception interface remains
entirely image-based, providing the motion measurements and temporal
associations required for relative-motion estimation without object
recognition or explicit obstacle-state or geometric reconstruction.

\subsection{Optimal-Decay Control Barrier Functions}
For $\dot{x}=f(x)+g(x)u$, let
$\mathcal{C}=\{x:h(x)\geq0\}$ be a safe set defined by a continuously
differentiable function $h$. Forward invariance is ensured when an
admissible input satisfies
$\dot{h}(x,u)+\gamma h(x)\geq0$, $\gamma>0$.
The optimal-decay CBF formulation \cite{zeng2021optimal} introduces
a decay multiplier $\Omega$ and jointly optimizes it with $u$:
$\dot{h}(x,u)+\Omega\gamma h(x)\geq0$,
yielding a quadratic program subject to actuator constraints.

\section{Concurrent Learning Observer}
\label{sec:cl}

The tracked image measurements are used to estimate the
depth-normalized relative velocity required for collision-cone
construction. We formulate this estimation using Concurrent Learning (CL),
where a finite history of image-space measurements enables online
estimation and certification of the estimation error under the stated
assumptions.

Let $\boldsymbol{x}=[x_1,x_2,x_3]^\top$ denote the obstacle position
in the camera frame and
$\boldsymbol{v}_{\mathrm{rel}}
=[v_{\mathrm{rel},1},v_{\mathrm{rel},2},v_{\mathrm{rel},3}]^\top$
the relative translational velocity. The normalized relative velocity is
defined as
\begin{equation}
\boldsymbol{\xi}
=
\frac{\boldsymbol{v}_{\mathrm{rel}}}{x_3}
=
\begin{bmatrix}
v_{\mathrm{rel},1}/x_3 &
v_{\mathrm{rel},2}/x_3 &
v_{\mathrm{rel},3}/x_3
\end{bmatrix}^{\!\top}.
\label{eq:xidef}
\end{equation}

The observer estimates $\boldsymbol{\xi}$ directly from image-space
measurements, without explicitly reconstructing obstacle depth or
relative velocity.

\subsection{Image-Space Relative Motion Model}
\label{sec:image-model}

Under the pinhole model, $\boldsymbol{x}$ projects onto the normalized image
plane as $y_1=x_1/x_3=(u-c_x)/f_x$ and $y_2=x_2/x_3=(v-c_y)/f_y$, with
principal point $(c_x,c_y)$ and focal lengths $(f_x,f_y)$. Let $\boldsymbol{y}=[y_1,y_2]^\top$. After rotational compensation discussed above, the compensated flow
$\boldsymbol{f}_{\mathrm{trans}}=[f_{\mathrm{trans},u},f_{\mathrm{trans},v}]^\top$
gives $\dot y_1=f_{\mathrm{trans},u}/(f_x\Delta t)$ and
$\dot y_2=f_{\mathrm{trans},v}/(f_y\Delta t)$, with $\Delta t$ the
sampling interval. The normalized image velocity then satisfies the
linear regression $\dot{\boldsymbol{y}}
=
\tilde{\boldsymbol{\phi}}(\boldsymbol{y})
\boldsymbol{\xi}$, with $\tilde{\boldsymbol{\phi}}(\boldsymbol{y})
= [1,\;0, \;-y_1;\;0,\;1,\;-y_2]$. This regression expresses the depth-normalized relative velocity
required for collision-cone construction directly in image-space
motion.

\subsection{Observer Design}
For frame $j$, stacking the normalized image-velocity measurements of the $n_j$ pixels within obstacle $i$'s mask gives
\begin{equation}
\boldsymbol{r}_j{=}\boldsymbol{\Phi}^{(j)}\boldsymbol{\xi}_j,\,
\boldsymbol{\Phi}^{(j)}{=}\big[\tilde{\boldsymbol{\phi}}(\boldsymbol{y}_j^{(1)})^\top\cdots
\tilde{\boldsymbol{\phi}}(\boldsymbol{y}_j^{(n_j)})^\top\big]^\top,
\label{eq:r-phi}
\end{equation}
where $r_j\in\mathbb{R}^{2n_j}$ stacks the measured normalized
image velocities, and $\Phi^{(j)}$ is the corresponding regression
matrix.
Estimating $\boldsymbol{\xi}$ from only the instantaneous regression
can be sensitive to image noise, temporary feature loss, and insufficient
excitation. To improve robustness, the proposed CL estimator augments the current measurement with a finite history of
informative image-space observations. Accordingly, the observer retains
the most recent $N$ frame blocks in a FIFO history stack, $\boldsymbol{R}^H_k=[\boldsymbol{r}_{k-N+1};\dots;\boldsymbol{r}_k]$,
$\boldsymbol{\Phi}^H_k=[\boldsymbol{\Phi}^{(k-N+1)};\dots;\boldsymbol{\Phi}^{(k)}]$. Since
$\boldsymbol{\xi}$ may vary over the window,
\begin{equation}
\boldsymbol{R}_k^{H}
=
\boldsymbol{\Phi}_k^{H}\boldsymbol{\xi}_k
+
\boldsymbol{E}_k,
\label{eq:stacked_regression}
\end{equation}
with
$\boldsymbol{E}_k=[\boldsymbol{\Phi}^{(k-N+1)}(\boldsymbol{\xi}_{k-N+1}-\boldsymbol{\xi}_k);\dots;
\boldsymbol{\Phi}^{(k-1)}(\boldsymbol{\xi}_{k-1}-\boldsymbol{\xi}_k);\boldsymbol{0}]$ capturing the
mismatch from the variation of $\boldsymbol{\xi}$ over the window. The estimate
minimizes the regularized cost
$C(\hat{\boldsymbol{\xi}}_k)=\|\boldsymbol{R}^H_k-\boldsymbol{\Phi}^H_k\hat{\boldsymbol{\xi}}_k\|^2
+\alpha\|\hat{\boldsymbol{\xi}}_k-\hat{\boldsymbol{\xi}}_{k-1}\|^2$ with $\alpha>0$,
giving
\begin{equation}
\hat{\boldsymbol{\xi}}_k
{=}
\hat{\boldsymbol{\xi}}_{k{-}1}
{+}
\left[
(\boldsymbol{\Phi}_k^{H})^\top
\boldsymbol{\Phi}_k^{H}
{+}
\alpha I
\right]^{-1}
(\boldsymbol{\Phi}_k^{H})^\top
\left(
\boldsymbol{R}_k^{H}
{-}
\boldsymbol{\Phi}_k^{H}
\hat{\boldsymbol{\xi}}_{k{-}1}
\right).
\label{eq:cl_update}
\end{equation}

\begin{remark}[FIFO History]
\label{rem:fifo}
A FIFO history is used instead of greedy excitation-maximizing
selection. Although greedy selection improves excitation, temporally
distant measurements may violate the bounded-variation assumption.
FIFO preserves temporal locality at the cost of non-monotone excitation.
\end{remark}

\subsection{Certified Estimation Error Bound}
\label{subsec:certified_bound}

\begin{assumption}[Finite-Time Excitation Condition]
\label{ass:PE}
There exist $\gamma_{\min},\gamma_{\max}>0$ and an integer $T_1>0$ with $\gamma_{\min} I
\preceq
(\boldsymbol{\Phi}_k^{H})^\top\boldsymbol{\Phi}_k^{H}
\preceq
\gamma_{\max} I, \forall k\geq T_1 .$
\end{assumption}

\begin{assumption}[Slowly Varying Relative Motion]
\label{ass:SV}
There exists $\delta_\xi>0$ : $\|\boldsymbol{\xi}_j-\boldsymbol{\xi}_k\|\le (k-j)\delta_\xi$ for $j=k-N+1,\ldots,k$.
\end{assumption}

\begin{remark}[Bounded History Disturbance]
\label{rem:disturbance}
Under Assumption~\ref{ass:SV}, $\|\boldsymbol{E}_k\|\le\bar E$ and
$\|(\boldsymbol{\Phi}^H_k)^\top \boldsymbol{E}_k\|\le\eta_E$ for some $\bar E,\eta_E>0$.
\end{remark}

\begin{theorem}[Certified Estimation Error Bound]
\label{thm:uub}
Under Assumptions~\ref{ass:PE}--\ref{ass:SV}, there exists
$\rho\in(0,1)$ such that the estimation error
$\tilde{\boldsymbol{\xi}}_{k}=\hat{\boldsymbol{\xi}}_{k}-\boldsymbol{\xi}_{k}$
satisfies, for all $k\geq T_1$,
\begin{equation}
\|\tilde{\boldsymbol{\xi}}_{k}\|
\leq
\rho^{k-T_1}\|\tilde{\boldsymbol{\xi}}_{T_1}\|
+
\frac{D\left(1-\rho^{k-T_1}\right)}{1-\rho}
=:B_c(k),
\label{eq:uub_bound}
\end{equation}
where $D=\rho\delta_\xi+\eta_E/(\gamma_{\min}+\alpha)$.
\end{theorem}

\begin{proof}
Substituting \eqref{eq:stacked_regression} into \eqref{eq:cl_update} and
writing $M_k=(\boldsymbol{\Phi}_k^H)^\top\boldsymbol{\Phi}_k^H+\alpha I$ gives the error
dynamics $\tilde{\boldsymbol{\xi}}_k
=
A_k\tilde{\boldsymbol{\xi}}_{k-1}
+
\boldsymbol{w}_k ,$
with
$A_k=I-M_k^{-1}(\boldsymbol{\Phi}_k^H)^\top\boldsymbol{\Phi}_k^H=\alpha M_k^{-1}$
and
$\boldsymbol{w}_k=-A_k(\boldsymbol{\xi}_k-\boldsymbol{\xi}_{k-1})
+M_k^{-1}(\boldsymbol{\Phi}_k^H)^\top\boldsymbol{E}_k$.
Under Assumption~\ref{ass:PE} the eigenvalues $\lambda_i$ of
$(\boldsymbol{\Phi}_k^H)^\top\boldsymbol{\Phi}_k^H$ lie in
$[\gamma_{\min},\gamma_{\max}]$, so those of $A_k$ are
$\mu_i=\alpha/(\lambda_i+\alpha)$ and
$|\mu_i|\leq\alpha/(\gamma_{\min}+\alpha)=:\rho<1$  since
$\gamma_{\min}>0$; hence $A_k$ is a contraction for all $k$. Using
Assumption~\ref{ass:SV}, Remark~\ref{rem:disturbance}, and
$\|M_k^{-1}\|\leq(\gamma_{\min}+\alpha)^{-1}$ yields
$\|\boldsymbol{w}_k\|\leq\rho\delta_\xi+\eta_E/(\gamma_{\min}+\alpha)=:D$, so
$\|\tilde{\boldsymbol{\xi}}_k\|\leq\rho\|\tilde{\boldsymbol{\xi}}_{k-1}\|+D$.
Recursion from $T_1$ to $k$ gives \eqref{eq:uub_bound}.
\end{proof}


\section{Vision-Based Collision Cone Control Barrier Function}
\label{sec:Vc3BF}

The CL observer provides the relative-motion estimate
$\hat{\boldsymbol{\xi}}$ and its error bound $B_c(k)$. Combined with
the image-derived line-of-sight direction, these define the nominal
collision cone, which is robustified to account for estimation
uncertainty and robot footprint.

\subsection{\texorpdfstring{Nominal Vision-Based Collision Cone (V-C$^3$BF)}
{Nominal Vision-Based Collision Cone (V-C3BF)}}
\label{subsec:NVc3BF}

The collision-cone principle characterizes risk through the direction of
relative motion with respect to the obstacle line of sight. Since the
collision cone is constructed in the horizontal plane $(\Pi)$, we define
$\hat{\boldsymbol{\xi}}_{\Pi} =
[\hat{\xi}_1\;\hat{\xi}_3]^\top$ and
$\boldsymbol{e}_{\mathrm{LOS},\Pi} =
[e_{\mathrm{LOS},1}\;e_{\mathrm{LOS},3}]^\top$. The angle $\theta$
between the relative-motion direction and the line of sight is given by
$\cos\theta =
-(\hat{\boldsymbol{\xi}}_{\Pi}^{\top}\boldsymbol{e}_{\mathrm{LOS},\Pi})/
(\|\hat{\boldsymbol{\xi}}_{\Pi}\|\|\boldsymbol{e}_{\mathrm{LOS},\Pi}\|)$.

The half-angle $\phi$ follows from the obstacle's apparent angular size
in the image, projected onto the same horizontal plane. The nominal
barrier is
\begin{equation}
h=\cos\phi-\cos\theta,
\qquad
\mathcal{C}=\{\hat{\boldsymbol{\xi}}_{\Pi}:h\geq0\},
\label{eq:vcbf}
\end{equation}
so $h\geq0$ corresponds to the estimated relative-motion direction lying outside the collision cone ($\theta\geq\phi$).
All quantities in the nominal barrier $h$ are obtained from monocular
image measurements, requiring neither metric obstacle state nor depth
sensing.

For multiple tracked obstacles, the nominal safe set is the intersection
of the per-obstacle safe sets \cite{molnar2023composing}. The
implementation selects
$i^\star=\arg\min_i h_i$, with $h_{\min}=h_{i^\star}$, and applies
the uncertainty and footprint margins to this obstacle to obtain
$h_r=h_{r,i^\star}$, which is enforced by the safety filter.

\subsection{\texorpdfstring{Robust Vision-Based Collision Cone (RV-C$^3$BF)}
{Robust Vision-Based Collision Cone (RV-C3BF)}}
\label{subsec:rvc3bf}

The nominal V-C$^3$BF assumes an exact relative-motion direction and a
point robot. The RV-C$^3$BF accounts for both effects by converting the
online estimation uncertainty into an angular margin and augmenting it
with a conservative robot-footprint margin.

Let $a_{\max}$ bound the
relative acceleration, $v_{\mathrm{rel},\max}$ the relative speed, $x_{3,\min}$ the
depth of any tracked obstacle, $\omega_{\max}$ the robot yaw rate, and $y_{\max}$ the normalized image coordinates,
$\|\boldsymbol y\|\le y_{\max}$. Then
$\|\boldsymbol\xi_k\|\le\xi_{\max}:=v_{\mathrm{rel},\max}/x_{3,\min}$ and
$\|\boldsymbol\xi_k-\boldsymbol\xi_{k-1}\|\le\delta_\xi$, where
$\delta_\xi=\Delta t_{\max}\big(a_{\max}/x_{3,\min}+\xi_{\max}^2
+\omega_{\max}(1+y_{\max})\,\xi_{\max}\big)$ and $\Delta t_{\max}$ is the
maximum interval between consecutive processed frames. For a FIFO window
holding $N_k\le N$ blocks, each containing at most $n_{\max}$ flow pixels
satisfying $\|\boldsymbol y\|\le y_{\max}$, define
$\bar E_k=\bar\Phi\,\delta_\xi\sqrt{(N_k-1)N_k(2N_k-1)/6}$, where
$\bar\Phi=\sqrt{n_{\max}(1+y_{\max}^2)}$. Let $\gamma_{\min}(k)$ and
$\gamma_{\max}(k)$ denote the extreme eigenvalues of
$(\boldsymbol\Phi_k^H)^\top\boldsymbol\Phi_k^H$, and define
$\rho_k=\alpha/(\gamma_{\min}(k)+\alpha)$ and
$D_k=\rho_k\delta_\xi+\sqrt{\gamma_{\max}(k)}\,\bar E_k/
(\gamma_{\min}(k)+\alpha)$. Each new or re-acquired track is initialized
with $\hat{\boldsymbol\xi}=\mathbf 0$ and $B_c=\xi_{\max}$, and the bound
is updated as $B_c(k)=\rho_kB_c(k-1)+D_k$. Since
$\|A_k\|=\rho_k\le1$ for every $k$, the error dynamics in the proof of
Theorem \ref{thm:uub} give
$\|\tilde{\boldsymbol\xi}_k\|\le\rho_k\|\tilde{\boldsymbol\xi}_{k-1}\|+D_k$,
so $B_c(k)$ bounds $\|\tilde{\boldsymbol\xi}_k\|$ under the measurement
model (\ref{eq:stacked_regression}) from track initialization onward. Since
projection onto $\Pi$ is non-expansive,
$\|\tilde{\boldsymbol\xi}_{\Pi,k}\|\le B_c(k)$, where
$\boldsymbol\xi_\Pi:=[\xi_1\;\xi_3]^\top$.

\subsubsection{Uncertainty-to-Angle Propagation}

Because the cone constrains a direction, the error bound must be
expressed in angular form. Let $\Delta\theta$ be the angular deviation
between the estimated and true relative-motion directions in $\Pi$.

\begin{lemma}[Maximum Angular Deviation]
\label{lem:angle}
Let $\tilde{\boldsymbol{\xi}}=\hat{\boldsymbol{\xi}}-\boldsymbol{\xi}$
satisfy $\|\tilde{\boldsymbol{\xi}}\|\le B_c(k)$ with
$\boldsymbol{\xi}_{\Pi}\neq\boldsymbol0$. If
$\|\hat{\boldsymbol{\xi}}_{\Pi}\|>B_c(k)$, then
\begin{equation}
\Delta\theta
\le
\arcsin\!\left(
\frac{B_c(k)}{\|\hat{\boldsymbol{\xi}}_{\Pi}\|}
\right)
=
\Delta\theta_{\max}.
\label{eq:lemma1-revised}
\end{equation}
\end{lemma}

\begin{proof}
By non-expansiveness of projection,
$\|\tilde{\boldsymbol{\xi}}_{\Pi}\|\le B_c(k)$. Since
$\boldsymbol{\xi}_{\Pi}
=\hat{\boldsymbol{\xi}}_{\Pi}
-\tilde{\boldsymbol{\xi}}_{\Pi}$ and
$\|\hat{\boldsymbol{\xi}}_{\Pi}\|>B_c(k)$, the angle
$\Delta\theta$ between the estimated and true directions belongs
to $[0,\pi/2)$. Let $\boldsymbol n$ be a unit vector orthogonal
to $\boldsymbol{\xi}_{\Pi}$. Then, $ \|\hat{\boldsymbol{\xi}}_{\Pi}\|\sin\Delta\theta
=
|\boldsymbol n^\top\hat{\boldsymbol{\xi}}_{\Pi}|
=
|\boldsymbol n^\top\tilde{\boldsymbol{\xi}}_{\Pi}|
\leq
B_c(k)$. Applying the monotonicity of $\sin(\cdot)$ on $[0,\pi/2)$
gives the result.
\end{proof}

Until $\|\hat{\boldsymbol{\xi}}_\Pi\|>B_c(k)$ is established, the obstacle is not certified; in this case, however, $|\xi_3|\le\|\boldsymbol{\xi}_\Pi\|\le 2B_c(k)$, so its time-to-contact is at least $1/(2B_c(k))$.

\subsubsection{Monocular Footprint Inflation}
\label{subsec:monocular_depth}

The finite robot footprint is incorporated through a conservative
monocular relative-depth estimate. Treating the obstacle as static, the optical expansion rate
gives $\xi_3=-\dot\phi/\phi=v_{rel,3}/x_3$, so $\hat x_3=-\boldsymbol{v}/\hat{\xi}_3$, where $\boldsymbol{v}$ is the robot velocity along the optical axis. The relative depth estimate $\hat{x}_3$, which is the robot optical-axis speed multiplied by the time-to-contact, equals the true depth $x_3$ only for static obstacles. The footprint margin uses a running minimum over the last $N$ relative-depth estimates to avoid transient overestimation. The robust cone half-angle is
\begin{align}
\phi_r &= \phi + \Delta\theta_{\max} + \phi_{\mathrm{robot}}, \nonumber\\
\phi+\phi_{\mathrm{robot}}
&=\arcsin\big(\min\{1,\sin\phi+r_{\mathrm{robot}}/\hat{x}_3\}\big)
,
\label{eq:phi_robust}
\end{align}
with $r_{\mathrm{robot}}$ the radius of a camera-centered disc containing the robot footprint. Under the static-obstacle
assumption the estimate satisfies $\hat{x}_3/x_3=v/(v-v_{\mathrm{obs}}),$
so an approaching or static obstacle ($v_{\mathrm{obs}}\leq 0$) yields
$\hat{x}_3\leq x_3$ and hence a conservative margin, while a receding obstacle ($v_{\mathrm{obs}}>0$) yields a smaller margin but a correspondingly longer time-to-contact. The robust barrier and
safe set are therefore $ h_r=\cos\phi_r-\cos\theta$ and
$\mathcal{C}_r=\{\hat{\boldsymbol{\xi}}_\Pi:h_r\geq0\},$ combining certified directional uncertainty with a conservative monocular
footprint margin directly in the cone geometry, with no external depth
sensor.

\subsection{Control Synthesis}
\label{subsec:control_synthesis}

A go-to-goal waypoint controller generates the nominal input command ($\boldsymbol{u}_{\mathrm{nom}}$). The RV-C$^3$BF then filters
this command through the optimal-decay CBF \cite{zeng2021optimal}:
\begin{align}
\min_{\boldsymbol{u}^c,\Omega}\quad
&
\tfrac{1}{2}(\boldsymbol{u^c}-\boldsymbol{u}_{\mathrm{nom}})^\top W(\boldsymbol{u^c}-\boldsymbol{u}_{\mathrm{nom}})
+
\tfrac{1}{2}p_\Omega(\Omega-\Omega_0)^2
\nonumber\\
\text{s.t.}\quad
&
\boldsymbol{L_f} h_r+\boldsymbol{L_g} h_r\,\boldsymbol{u^c} \geq -\Omega\,\gamma h_r, \nonumber\\
&
\boldsymbol{u}_{\min}\leq \boldsymbol{u^c}\leq \boldsymbol{u}_{\max}, \qquad
0\leq\Omega\leq\bar{\Omega},
\label{eq:odcbf}
\end{align}
where $\boldsymbol{W}\succ0$ penalizes deviation from the nominal command,
$p_\Omega>0$ regulates the decay multiplier $\Omega$, and $\gamma>0$
is the CBF gain. The optimal-decay variable relaxes the prescribed
decay rate when necessary while remaining bounded. Thus, the
filter modifies the nominal command only as required to satisfy the
robust safety constraint. Fig. \ref{fig:BD} and Algorithm \ref{alg:pipeline} summarize the resulting perception-to-control loop.

\section{Application to an Acceleration-Controlled Unicycle}
\label{sec:unicycle}

The RV-C$^3$BF is implemented on the acceleration-controlled unicycle
model of the TurtleBot3. Under the constant-velocity obstacle assumption,
we show that the barrier has relative degree one with respect to the
acceleration input and can therefore be enforced directly using the
OD-CBF QP. The metric model below is used only for this analysis; all
quantities required by the controller are obtained from monocular
measurements and proprioception.

Consider the planar robot state
$\boldsymbol{x}=[x_p,y_p,\theta_c,v,\omega_c]^\top$ and input
$\boldsymbol{u^c}=[a_c,\alpha_c]^\top$, with the camera mounted a distance
$l$ ahead of the robot reference point. Let
$c=\cos\theta_c$ and $s=\sin\theta_c$. The camera position in the world frame is
$\boldsymbol{p}_c^W=[x_p+lc,\;y_p+ls]^\top$. For a constant-velocity
obstacle, the relative position and velocity are
$\boldsymbol{p}_{\rm rel}^W=\boldsymbol{p}_{\rm obs}^W-\boldsymbol{p}_c^W$
and
$\boldsymbol{v}_{\rm rel}^W=\dot{\boldsymbol{p}}_{\rm obs}^W-
\dot{\boldsymbol{p}}_c^W$, giving
$\dot{\boldsymbol{v}}_{\rm rel}^W=F+G\boldsymbol{u^c}$, where
$F=[v\omega_cs+l\omega_c^2c,\;-v\omega_cc+l\omega_c^2s]^\top$ and
$G=[-c,\;ls;\;-s,\;-lc]$ \cite{tayal2026collision}.

In the horizontal camera plane, with axes $x_1$ (right) and $x_3$
(forward), the world-to-camera transformation is
$R(\theta_c)= [s,\;-c;\;c,\;s]$, such that
$\boldsymbol{p}_{\rm rel}=R\boldsymbol{p}_{\rm rel}^W$ and
$\boldsymbol{v}_{\rm rel}=R\boldsymbol{v}_{\rm rel}^W$. Since
$\dot{\theta}_c=\omega_c$, $\dot R=\omega_cSR$, where
$S=[0,\;1;\;-1,\;0]$. Hence,
$\dot{\boldsymbol{p}}_{\rm rel}=\omega_cS\boldsymbol{p}_{\rm rel}
+\boldsymbol{v}_{\rm rel}$, where the first term is rotational image
motion removed during optical-flow de-rotation. Differentiating the
relative velocity gives
$\dot{\boldsymbol{v}}_{\rm rel}=\omega_cS\boldsymbol{v}_{\rm rel}
+F_c+G_c\boldsymbol{u^c}$, with
$F_c=[v\omega_c,\;l\omega_c^2]^\top$ and $G_c=[0,\;l;\;-1,\;0]$.

Let $\boldsymbol{e}= \hat{\boldsymbol{e}}_{\mathrm{LOS},\Pi} = \boldsymbol{p}_{\rm rel}/
\|\boldsymbol{p}_{\rm rel}\|$ and
$\boldsymbol{q}=\boldsymbol{v}_{\rm rel}/
\|\boldsymbol{v}_{\rm rel}\|$. The collision angle satisfies
$\cos\theta=-\boldsymbol{q}^\top\boldsymbol{e}$, yielding the robust
collision-cone barrier
$h_r=\cos\phi_r+\boldsymbol{q}^\top\boldsymbol{e}$.
The corresponding LOS and relative-motion dynamics are
\[
\dot{\boldsymbol e}
{=}
\omega_cS\boldsymbol e+
\frac{(I{-}\boldsymbol e\boldsymbol e^\top)\boldsymbol v_{\rm rel}}
{\|\boldsymbol p_{\rm rel}\|}, 
\dot{\boldsymbol q}
{=}
\omega_cS\boldsymbol q+
\frac{(I{-}\boldsymbol q\boldsymbol q^\top)
(F_c{+}G_c\boldsymbol u^c)}
{\|\boldsymbol v_{\rm rel}\|}.
\]

Differentiating the barrier gives
$\dot h_r=L_fh_r+L_gh_r\boldsymbol u^c$, where
$L_gh_r=
\boldsymbol e^\top(I-\boldsymbol q\boldsymbol q^\top)G_c/
\|\boldsymbol v_{\rm rel}\|$ and
$L_fh_r=-\sin\phi_r\,\dot\phi_r+
\boldsymbol q^\top\dot{\boldsymbol e}+
\boldsymbol e^\top\left[
\omega_cS\boldsymbol q+
(I-\boldsymbol q\boldsymbol q^\top)F_c/
\|\boldsymbol v_{\rm rel}\|\right]$.

\begin{proposition}[Relative Degree]
\label{prop:relative_degree}
For a fixed instantaneous robust-cone geometry with $0<\phi_r<\pi$,
the RV-C$^3$BF has relative degree one with respect to the
acceleration-controlled unicycle input.
\end{proposition}

\begin{proof}
Since $\det(G_c)=l>0$, $L_gh_r=0$ only if
$\boldsymbol q=\pm\boldsymbol e$. On $\partial\mathcal C_r$,
$\boldsymbol q^\top\boldsymbol e=-\cos\phi_r$; hence
$\boldsymbol q=\boldsymbol e$ requires $\phi_r=\pi$, while
$\boldsymbol q=-\boldsymbol e$ requires $\phi_r=0$. Thus,
$L_gh_r\neq0$ for $0<\phi_r<\pi$, establishing relative degree one.
\end{proof}

\begin{corollary}[Closed-Loop Safety]\label{cor:safety}
Model $i^\star$ as a disc of radius $r_{\mathrm{obs}}$ centered at $\boldsymbol{p}_{\mathrm{rel}}=[x_1\;x_3]^\top$, and let $d=\|\boldsymbol{p}_{\mathrm{rel}}\|$ and $\bar{d}=r_{\mathrm{obs}}+r_{\mathrm{robot}}$. Suppose that, on $[t_0,t_1)$, $i^\star$ is fixed; (\ref{eq:odcbf}) is feasible and applied continuously, $\dot{h}_r=L_fh_r+L_gh_r\boldsymbol{u}^c$ between perception updates, and $h_r\ge 0$ at $t_0$ and after each update; the hypotheses of Lemma~\ref{lem:angle} hold; $\Delta\theta_{\max}$ and $\phi_{\mathrm{robot}}$ are no smaller than in (\ref{eq:lemma1-revised})--(\ref{eq:phi_robust}), with $\phi_r<\pi$; and $\boldsymbol{e}_{\mathrm{LOS},\Pi}$ is exact, $\phi=\arcsin(\min\{1,r_{\mathrm{obs}}/d\})$, and $0<\hat{x}_3\le x_3$. If $d(t_0)\ge\bar{d}$, then $d\ge\bar{d}$ on $[t_0,t_1)$, i.e., the robot does not collide with $i^\star$.
\end{corollary}

\begin{proof}
By (\ref{eq:odcbf}), $\dot{h}_r\ge-\Omega\gamma h_r$ between updates with $\Omega$ bounded, so the comparison lemma gives $h_r\ge 0$ and hence $\theta\ge\phi_r$, as $\theta,\phi_r\in[0,\pi]$. Let $\theta^*$ be $\theta$ evaluated with the true $\boldsymbol{\xi}_\Pi=\boldsymbol{v}_{\mathrm{rel}}/x_3$. By Lemma~\ref{lem:angle}, the triangle inequality for angles, and (\ref{eq:phi_robust}) with $\hat{x}_3\le x_3\le d$ and the expression for $\phi$, $\theta^*\ge\theta-\Delta\theta_{\max}\ge\phi+\phi_{\mathrm{robot}}\ge\arcsin(\min\{1,\bar{d}/d\})$. Thus $d\le\bar{d}$ implies $\theta^*\ge\pi/2$ and $\dot{d}=\boldsymbol{e}^\top\boldsymbol{v}_{\mathrm{rel}}=-\|\boldsymbol{v}_{\mathrm{rel}}\|\cos\theta^*\ge 0$ (as $\boldsymbol{e}^\top S\boldsymbol{e}=0$), so $d$ cannot fall below $\bar{d}$.
\end{proof}



\begin{algorithm}[t]
\caption{RV-C$^3$BF perception-to-control loop}
\label{alg:pipeline}
\begin{algorithmic}[1]
\STATE Acquire image; compute Lucas--Kanade flow $\boldsymbol{f}_{\text{tot}}$
\STATE De-rotate: $\boldsymbol{f}_{\text{trans}}\leftarrow \boldsymbol{f}_{\text{tot}} - \boldsymbol{f}_{\text{rot}}$ using proprioception 
\STATE Segment independently moving regions (DBSCAN + GrabCut); update tracks $\{\mathcal{P}_i\}$ 
\FOR{each tracked obstacle $i$}
  \STATE Form regression block $\boldsymbol{\Phi}^{(k)}_i$, $\boldsymbol{r}^{(k)}_i$ from $\mathcal{P}_i$ \eqref{eq:r-phi}
  \STATE Update FIFO history $(\boldsymbol{R}^H_{k,i},\boldsymbol{\Phi}^H_{k,i})$; update $\hat{\boldsymbol{\xi}}_{k,i}$ \eqref{eq:cl_update}
  \STATE Evaluate $\gamma_{\min},\gamma_{\max}$; compute $\boldsymbol{B_c(k)}$ 
  \STATE Compute nominal barrier $h_i$ \eqref{eq:vcbf}
\ENDFOR
\STATE Select critical obstacle:
$i^\star\leftarrow\arg\min_i h_i$, \quad $h_{\min}\leftarrow h_{i^\star}$
\STATE For obstacle $i^\star$, update running-minimum relative depth $\hat{x}_3$;
compute $\Delta\theta_{\max}$, $\phi_{\mathrm{robot}}$, $\phi_r$ \eqref{eq:phi_robust}
\STATE Form robust barrier $h_r=\cos\phi_r-\cos\theta$ 
\STATE Solve OD-CBF QP \eqref{eq:odcbf} for $u^c$; apply $\boldsymbol{u}^c$ to the robot
\end{algorithmic}
\end{algorithm}

\begin{figure}[t]
    \centering
    \includegraphics[width=\columnwidth]{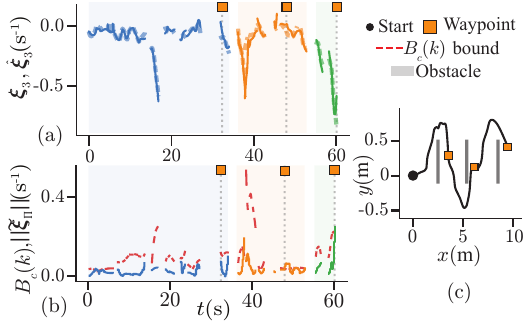}
    \caption{Mission-wide CL observer validation across three sequential
    static-obstacle encounters at $(2.5,0.2)$, $(5.5,0.2)$, and
    $(8.5,0.2)\,\mathrm{m}$. \textbf{(a)} Estimated $\hat{\xi}_3$ (solid)
    versus ground truth (dashed). \textbf{(b)} In-plane estimation error and online
    bound $B_c(k)$ (red, dashed). \textbf{(c)} Resulting trajectory through
    waypoints $(3.5,0.3)$, $(6.0,0.1)$, and $(9.5,0.4)\,\mathrm{m}$. Blue,
    orange, and green shading indicate obstacles 0, 1 and 2; gray dotted lines
    mark waypoint transitions.}
    \label{fig:cl-mission-validation}
\end{figure}

\begin{figure*}[t]
\includegraphics[width=\textwidth]{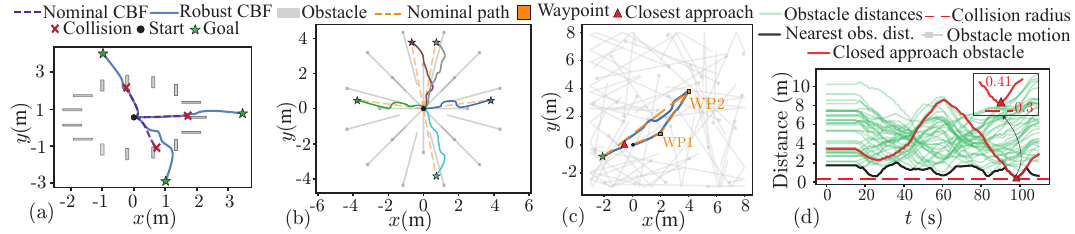}
\caption{\textbf{(a)} Nominal V-C$^3$BF vs.\ robust RV-C$^3$BF escaping a
15-obstacle ring: the nominal controller
collides in all three trials (red $\times$), the proposed method
succeeds in all three. \textbf{(b)} 20 obstacles converging on the
robot's start position; five commanded goals, all reached
collision-free. \textbf{(c)} Waypoint following through 60 moving
obstacles via $(2.0,0.8)$ and $(4.0,3.8)$ to goal $(-2.0,-0.8)$.
\textbf{(d)} Distances to all 60 obstacles (green) and to the nearest obstacle (black) in (c). The red solid curve shows the closest-approach obstacle, and the red dashed line marks the $0.30\,\mathrm{m}$ collision radius.}
\label{fig:results_sim1}
\end{figure*}

For implementation, the relative motion direction and relative-speed magnitude
are obtained from the CL estimate as
$\hat{\boldsymbol q}=\hat{\boldsymbol\xi}_{\Pi}/\|\hat{\boldsymbol\xi}_{\Pi}\|$ and
$\|\boldsymbol v_{\rm rel}\|=x_3\|\boldsymbol\xi_{\Pi}\|
\approx\hat x_3\|\hat{\boldsymbol\xi}_{\Pi}\|$. Thus,
$L_gh_r=
\boldsymbol e^\top(I-\hat{\boldsymbol q}\hat{\boldsymbol q}^\top)G_c/
(\|\hat{\boldsymbol\xi}_\Pi\|\hat x_3)$ and
$L_fh_r=-\sin\phi_r\,\dot\phi_r+
\hat{\boldsymbol q}^\top\dot{\boldsymbol e}+
\boldsymbol e^\top\left[
\omega_cS\hat{\boldsymbol q}+
(I-\hat{\boldsymbol q}\hat{\boldsymbol q}^\top)F_c/
(\|\hat{\boldsymbol\xi}_{\Pi}\|\hat x_3)\right]$.
Between perception updates, $\boldsymbol{e}$ and $\hat{\boldsymbol{q}}$ are propagated with the dynamics above, while the uncertainty and footprint margins are held constant, with $\dot{\phi}=-\hat{\xi}_3\phi$.

\section{Validation}
The validation is structured to test the central hypothesis progressively:
the scale-free relative-motion representation is first recovered from
image-space measurements, its uncertainty is quantified, the resulting
representation is used for robust safety certification, and the complete
framework is finally evaluated on a real robot.

\subsection{Simulation Results}
\label{subsec:simulation}

The framework is implemented in Python/ROS~2 and evaluated in Gazebo
using a TurtleBot3 with a forward-facing monocular camera ($80^\circ$
field of view), mounted $l=76$\,mm ahead of the axle, and $r_{\mathrm{robot}}=0.09$\,m. Unless stated otherwise, simulations
use $N=6$ frames, $\gamma=1$, and acceleration limits
$\boldsymbol{u}_{\min}=[-0.5,-1.5]^\top$ and
$\boldsymbol{u}_{\max}=[0.5,1.5]^\top$, in m/s$^2$ and rad/s$^2$,
respectively. 

\subsubsection{Concurrent Learning Observer Validation}
\label{sec:cl-validation}

Fig.~\ref{fig:cl-mission-validation} evaluates the CL observer across
three static-obstacle encounters, including brief field-of-view loss
and reacquisition. As shown in
Fig.~\ref{fig:cl-mission-validation}(a), the estimated looming
component $\hat{\xi}_3$ closely follows the ground truth, with an RMSE of
$0.0780\,\mathrm{s}^{-1}$. The scale-free relative-motion
representation required for collision-cone construction is accurately
recovered directly from image-space measurements. Fig.~\ref{fig:cl-mission-validation}(b)
shows the in-plane estimation error and online uncertainty bound
$B_c(k)$, with the bound containing the measured error in $98.3\%$ of
the evaluated samples. The remaining $1.7\%$ of violations reflect unmodeled uncertainties arising from obstacle segmentation and optical flow measurement noise.
While (\ref{eq:stacked_regression}) accounts for within-window variation in $\boldsymbol{\xi}$,
it does not explicitly capture optical-flow errors, numerical
differentiation, de-rotation residuals, or depth variation and
segmentation leakage. The high containment rate of the
online uncertainty bound validates its use for online robust cone inflation. A six-frame history achieved the lowest estimation RMSE
among the tested window lengths and was used in subsequent experiments. Fig.~\ref{fig:cl-mission-validation}(c)
shows the corresponding navigation trajectory, where the robot
successfully deflects around the three obstacles and resumes its
waypoint path following each field-of-view interruption.

\subsubsection{Robust Collision-Cone Validation}

\begin{figure*}[t]
\centering
\includegraphics[width=\linewidth]{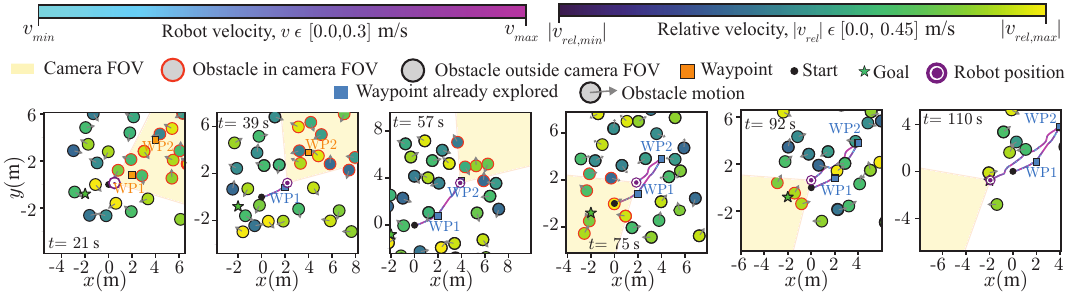}
\caption{Snapshots of the waypoint-following mission of
Fig.~\ref{fig:results_sim1}(c). The purple trace shows the path traversed so far, colored by robot's
instantaneous speed ($v$), illustrating the online adaptation of the
RV-C$^3$BF as it continuously deflects around moving obstacles while
maintaining progress toward the active waypoint.}
\label{fig:snapshot}
\end{figure*}

Fig.~\ref{fig:results_sim1}(a) compares the proposed RV-C$^3$BF with
the nominal vision-based collision cone using identical perception,
obstacle configuration, and nominal control. The robot starts inside a
ring of 15 static obstacles and is commanded toward goals outside the
ring. Across three trials for each controller, the nominal controller
collides with the ring in every trial, whereas the RV-C$^3$BF safely
threads through the obstacles and reaches the goal in all trials. This
experiment isolates the effect of uncertainty propagation and footprint inflation, since
perception, obstacle configuration, and nominal control remain
unchanged. The improved behavior results from incorporating the
certified perception uncertainty and robot-footprint margin into the
collision-cone geometry, enabling earlier safety intervention.

\subsubsection{Dynamic Obstacle Avoidance}
\label{subsec:dynamic_obstacle_avoidance}

The proposed RV-C$^3$BF is evaluated in dynamic environments, where
the collision cone is updated online from the estimated normalized
relative motion without prior knowledge of obstacle trajectories. The
robot speed is limited to $0.3\,\mathrm{m/s}$. Fig.~\ref{fig:results_sim1}(b)
shows 20 obstacles moving radially toward the robot's initial position
from different directions. For each of five commanded goals, the robot
deviates from the nominal path only when required for collision
avoidance and reaches the goal.

Fig.~\ref{fig:results_sim1}(c) extends the evaluation to a
longer-horizon mission with 60 obstacles moving at distinct constant
velocities, with the robot sequentially tracking waypoints
$(2.0,0.8)$ and $(4.0,3.8)$ before reaching the final goal at
$(-2.0,-0.8)$ (snapshots in
Fig.~\ref{fig:snapshot}). Fig.~\ref{fig:results_sim1}(d) shows the
resulting clearance throughout the mission. The distance to every
obstacle (green) and the distance to the nearest obstacle at each
instant (black) both remain above the $0.30\,\mathrm{m}$ collision
radius, with a minimum observed clearance of $0.41\,\mathrm{m}$.

These results demonstrate that the proposed framework can update
collision geometry online and maintain safe navigation in cluttered
dynamic environments using monocular vision and onboard proprioception.

\begin{table}[t]
\centering
\caption{Average processing time of the vision and control stages of the implementation pipeline}
\label{tab:runtime}
\small
\setlength{\tabcolsep}{3pt}
\begin{tabular}{lcccccc}
\hline
& \textbf{Opt. Flow} & \textbf{Seg.} & \textbf{CL}
& \textbf{Cone} & \textbf{CBF-QP} & \textbf{Ctrl.} \\
\hline
\textbf{Time (ms)}
& 7.36 & 11.83 & 2.60 & 0.93 & 0.26 & 0.36 \\
\hline
\end{tabular}
\end{table}

\subsubsection{Runtime Performance}
\label{sec:runtime}

Table~\ref{tab:runtime} summarizes the execution time of the vision and control stages. The perception pipeline is
computationally dominated by motion segmentation, which includes the
obstacle-mask refinement step, while the CL update and collision-cone
evaluation remain lightweight. The perception and control loops run
asynchronously: intermediate camera frames are discarded while perception is busy, preventing backlogs and allowing the $50$-Hz
control loop to operate independently of the lower-rate perception
pipeline.

\subsection{Experimental Setup}
\label{sec:experimental_setup}

The proposed framework is evaluated on a TurtleBot3 equipped with an
Intel RealSense camera operated in monocular RGB mode, with its depth
stream disabled. All perception, normalized relative-motion estimation,
and control computations are performed onboard an NVIDIA Jetson Orin
Nano. The system relies exclusively on monocular vision and
proprioceptive odometry, without external sensing or offboard
computation. Obstacle states and 3D geometry are neither provided to the controller nor explicitly reconstructed; instead, collision-cone constraints are constructed directly from monocular image measurements and depth-normalized relative-motion estimates. The maximum
achievable forward speed of the experimental TurtleBot3 platform is
$0.22\,\mathrm{m/s}$.

The experiments evaluate the sufficiency of the proposed representation
for safety certification rather than serving as a navigation-performance
benchmark. Experiments are conducted in two environments (Fig.~\ref{fig:results_exp}):
(a) an open laboratory space and (b) a narrow corridor, both populated
with static and dynamic obstacles. Static obstacles are represented by
boxes, while a second TurtleBot3 serves as the dynamic obstacle, moving
at constant velocities between $0.09$ and $0.12\,\mathrm{m/s}$. Initial conditions,
goal positions, obstacle configurations, and dynamic-obstacle motions
are randomized across trials in both environments. The 100 hardware
trials therefore constitute a Monte Carlo study under varying
experimental conditions.

\subsubsection{Runtime Performance and Discussion}
\label{sec:runtime_exp}

The onboard Jetson Orin Nano executes the perception pipeline at
approximately $15\,\mathrm{Hz}$, while the CBF-QP and control loops
operate at $50\,\mathrm{Hz}$. The asynchronous architecture allows the
controller to operate independently of perception updates, using the
latest available collision-cone representation. The system avoided collisions in 93 out of 100 hardware trials;
the seven failures occurred when obstacles entered the camera field
of view at very close range, leaving insufficient time for perception
and avoidance. No OD-CBF-QP infeasibility
(\ref{eq:odcbf}) was observed. These results identify late perception,
rather than QP infeasibility, as the observed failure mechanism.

The safety certificate applies to the selected critical obstacle,
with the cone margins
$\Delta\theta_{\max}$ and $\phi_{\mathrm{robot}}$ held constant
between perception updates. Its validity also relies on Assumption~1,
which requires a sufficiently informative history stack. Consequently,
the reported guarantees do not establish safety for obstacles that
appear too close for the required perception and estimation conditions
to be met. Improving perception throughput and handling sudden
obstacle appearances are therefore important directions for improving
real-world performance.

\begin{figure*}[t]
    \centering
    \includegraphics[width=\linewidth]{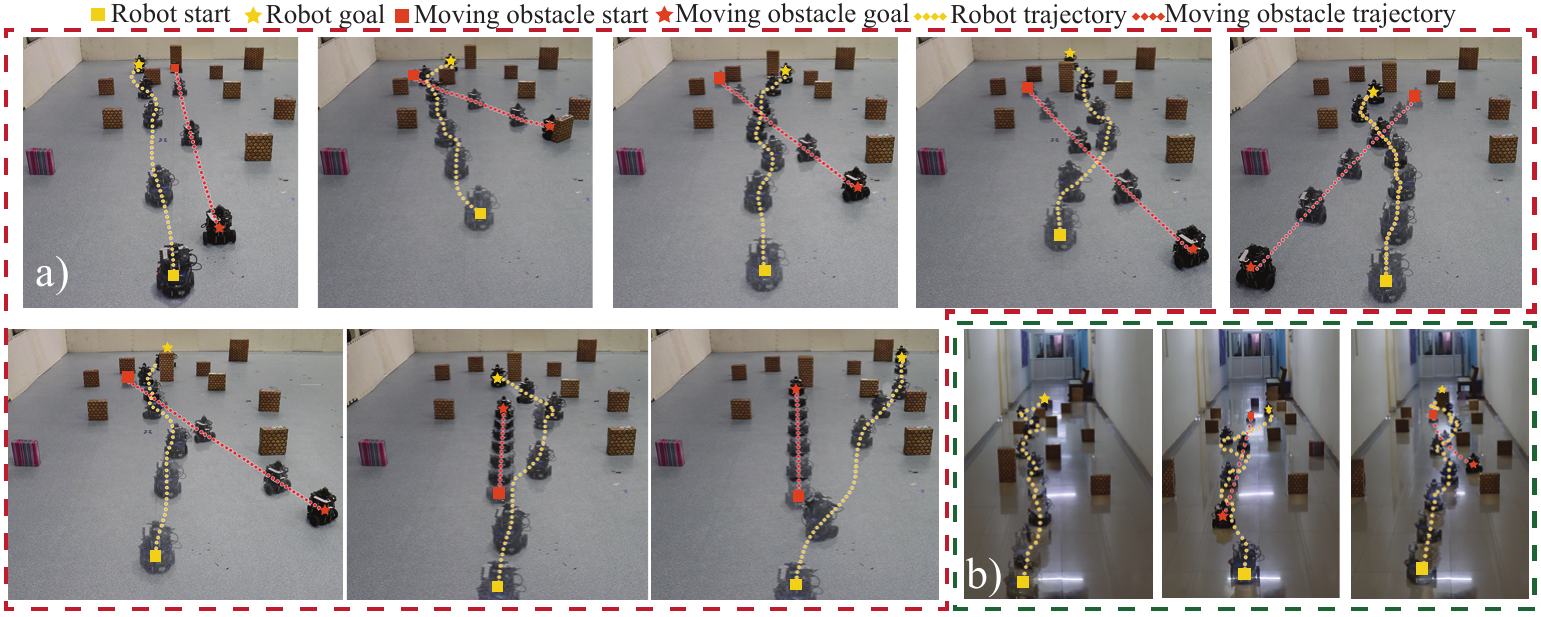}
    \caption{Experimental validation of the proposed RV-C$^3$BF under different initial and final conditions, demonstrating collision avoidance and safe navigation in cluttered static and dynamic environments: a) open laboratory space, and b) narrow corridor.}
    \label{fig:results_exp}
\end{figure*}

\subsection{Comparison with Existing Approaches}
\label{subsec:comparison}

Table~\ref{tab:comparison} summarizes representative vision-based
approaches in terms of sensing, safety certification, obstacle-state
reconstruction, and dynamic-obstacle validation. To the best of our
knowledge, this framework is among the first to combine onboard
monocular perception with formal safety certification without explicit
obstacle-state reconstruction, with validation under independent
obstacle motion.

For a direct comparison, we evaluate the full-state C$^3$BF
\cite{tayal2026collision} under identical Gazebo scenarios, nominal
control, and OD-CBF settings. Enforcing all collision-cone constraints
renders the baseline QP infeasible in every scenario in
Fig.~\ref{fig:results_sim1}, consistent with
\cite{park2026beyondieee,sah2026adversarially}. When restricted to the
selected critical obstacle, however, both C$^3$BF and the proposed
RV-C$^3$BF remain feasible and achieve collision-free navigation.
This comparison highlights the role of critical-obstacle selection in
maintaining feasibility and shows that the proposed framework can
achieve collision-free navigation using scale-free relative motion
from monocular vision and proprioception, without explicit
obstacle-state reconstruction.

\begin{table}[t]
\centering
\caption{Comparison of representative approaches for vision-based
collision avoidance and safety certification.}
\label{tab:comparison}
\scriptsize
\setlength{\tabcolsep}{2pt}
\renewcommand{\arraystretch}{1.05}
\begin{tabular}{@{}p{0.25\columnwidth}
                p{0.15\columnwidth}
                p{0.14\columnwidth}
                p{0.21\columnwidth}
                p{0.17\columnwidth}@{}}
\hline
\textbf{Approach} &
\textbf{Sensing} &
\textbf{Safety} &
\textbf{State/Map} &
\textbf{Dynamic Val.} \\
\hline
Optical flow
\cite{deng2026autonomous,hu2025seeing}
& Mono. & No & No & No \\

C$^3$BF
\cite{tayal2026collision}
& Full state & Yes & Required & Yes \\

RGB-D CBF
\cite{abdi2023safe}
& RGB-D & Yes & Required & Yes$^{*}$ \\

Point-cloud CBF
\cite{de2024point}
& Depth/LiDAR & Yes & Required & Yes$^{**}$ \\

VLM-tuned CBF
\cite{chen2026dynamic}
& RGB$^{\ddagger}$ & Yes & Required & Yes$^{***}$ \\

Bioinspired SF
\cite{escobar2019bioinspired}
& Mono. & No & No & No \\

Semantic CBF
\cite{zhang2026embedding}
& Mono. & Yes & Required & No \\

\textbf{Proposed}
& \textbf{Mono.} & \textbf{Yes}
& \textbf{No} & \textbf{Yes} \\
\hline
\end{tabular}

\vspace{2pt}
\parbox{\columnwidth}{\scriptsize
Mono.: monocular. State/Map indicates reliance on external depth
measurements, explicit 3D geometry, reconstructed maps, or explicit
obstacle states, as applicable.
$^{*}$ Dynamic validation in CARLA.
$^{**}$ Validation limited to one hardware run per platform. $^{\ddagger}$ RGB sets the CBF gain via a VLM; the barrier uses
known obstacle centers and radii.
$^{***}$ Dynamic validation in Isaac Sim.
}
\end{table}

\section{Conclusion}

This letter presented a visual certification framework that constructs collision-risk certificates directly from optical flow and onboard
proprioception. The key insight is that depth-normalized relative
motion provides a scale-free representation for collision-cone
construction without explicit obstacle-state or 3D-environment
reconstruction. A Concurrent Learning observer estimates this relative
motion and provides an online estimation-error bound, which is
propagated into a robust collision cone incorporating a robot-footprint
margin. The resulting RV-C$^3$BF is enforced through an optimal-decay
CBF-QP and evaluated in simulation and on a TurtleBot3 across static,
dynamic, and long-horizon navigation scenarios. Although demonstrated
on a differential-drive platform, the framework may be adapted to
other acceleration-controlled vehicles, subject to their dynamics and
sensing assumptions. Future work will investigate event-based and
omnidirectional cameras to alleviate the perception bottleneck and
extend the framework to agile platforms such as quadrotors. 

\bibliographystyle{IEEEtran}
\bibliography{References}
\end{document}